\documentclass[twoside]{article}
\usepackage[margin=1in]{geometry}

\usepackage{times}
\usepackage{appendix}

\usepackage[natbib=true,style=authoryear-comp,doi=false,isbn=false,url=false]{biblatex}
\bibliography{SDMbib.bib}

\RequirePackage{amsmath,amssymb,amsthm,bm,bbm,fullpage,graphicx,psfrag,amsfonts,verbatim,mathtools,upgreek,xcolor}
\numberwithin{equation}{section}
\usepackage{subcaption}
\usepackage{thmtools}
\usepackage{thm-restate}
\usepackage{hyperref}
\usepackage{cleveref}
\usepackage[ruled]{algorithm2e}
\newtheorem*{theorem*}{Theorem}

\DeclarePairedDelimiter{\ceil}{\lceil}{\rceil}
\DeclarePairedDelimiter{\floor}{\lfloor}{\rfloor}

\newcommand{\bigo}{\mathcal{O}}
\newcommand{\EX}{\mathbb{E}}
\newcommand\1{\mathbf{1}}

\DeclareMathOperator*{\argmin}{argmin}

\title{What You See May Not Be What You Get: \\
UCB Bandit Algorithms Robust to $\varepsilon$-Contamination}

\author{Laura Niss, Ambuj Tewari \\ University of Michigan}  % LEAVE BLANK FOR ORIGINAL SUBMISSION.
\date{}
\begin{document}
\twocolumn
\maketitle

\begin{abstract}
Motivated by applications of bandit algorithms in education, we consider a stochastic multi-armed bandit problem with $\varepsilon$-contaminated rewards.
We allow an adversary to give arbitrary unbounded contaminated rewards with full knowledge of the past and future.
We impose the constraint that for each time $t$ the proportion of contaminated rewards for any action is less than or equal to $\varepsilon$.
We derive concentration inequalities for two robust mean estimators for sub-Gaussian distributions in the $\varepsilon$-contamination context.
We define the $\varepsilon$-contaminated stochastic bandit problem and use our robust mean estimators to give two variants of a robust Upper Confidence Bound (UCB) algorithm, crUCB. Using regret derived from only the underlying stochastic rewards, both variants of crUCB achieve $\bigo (\sqrt{KT\log T})$ regret for small enough contamination proportions.
Our simulations assume small horizons, reflecting the newly explored setting of bandits in education. We show that in certain adversarial regimes crUCB not only outperforms algorithms designed for stochastic (UCB1) and adversarial (EXP3) bandits but also those that have ``best of both worlds" guarantees (EXP3++ and TsallisInf) even when our constraint on the proportion of contaminated rewards is broken.

%   Motivated by applications of bandits algorithms in education, where rewards are often human response,  we study bandit problems with contamination in the rewards. It is well known that people do not always give thoughtful or effortful feedback. To account for this problem, we present algorithms for stochastic multi-armed bandits in the presence of adversarial contamination in  rewards.
%   We do this by proposing simple contamination robust statistics and give their concentration bounds that hold under an $\varepsilon$-contamination model. Using these contamination robust statistics, we present a contamination robust-UCB algorithm, with guarantees for rewards with sub-Gaussian true underlying distributions. We also argue for a different notion of regret in the contaminated bandits setting that measures performance of an algorithm only on the true, uncontaminated rewards, and not the possibly contaminated observed rewards.
\end{abstract}

%%%%%%%%%%%%%%%%%%%%%%%%%%%%%%%%%%%%%%%%%%%%%%%%%%%%%%%%%%%%%%%%%%%%%%%%%%%%%%%%

\section{INTRODUCTION}
We first review the problem of stochastic multi-armed bandits (sMAB) with contaminated rewards, or contaminated stochastic bandits (CSB). This scenario assumes that rewards associated with an action are sampled i.i.d. from a fixed distribution and that the learner observes the reward after an adversary has the opportunity to contaminate it. The observed reward can be unrelated to the reward distribution and can be maliciously chosen to fool the learner. An outline for this setup is presented in \Cref{sec:Problem Setting}.
%This problem has recently started to gain attention within the bandit literature, with varying assumptions on the adversary and rewards \cite{altschulerBestArmIdentification2018, lykourisStochasticBanditsRobust2018, kapoorCorruptiontolerantBanditLearning2018, junAdversarialAttacksStochastic2018}.

We are primarily motivated by the use of bandit algorithms in education, where the rewards often come directly from human opinion. Whether responses come from undergraduate students, a community sample, or paid participants on platforms like MTurk, there is always reason to believe some responses are careless or inattentive to the question or could be assisted by bots \citep{neckaMeasuringPrevalenceProblematic2016, curranMethodsDetectionCarelessly2016}.

An example in education is a recent paper testing bandit Thompson sampling to identify high quality student generated solution explanations to math problems using MTurk participants \citep{williamsAXISGeneratingExplanations2016}. Using a rating between 1-10 from 150 participants, the results showed that Thompson sampling identified participant generated explanations that when viewed by other participants significantly improved their chance of solving future problems compared to no explanation or ``bad" explanations identified by the algorithm. While the proportion of contaminated responses will always depend on the population, recent work suggests even when screening out fraudulent participants, between $2-30\%$ of MTurk participants give low-quality samples \citep{ahlerMicroTaskMarketLemons2018, ryanDataContaminationMTurk2018, neckaMeasuringPrevalenceProblematic2016}. This is consistent with measurements of careless and inattentive responses seen in survey data, which reports $1-30\%$ with an estimated mode of $8-12\%$, with the conclusion that these responses are generally not a random sample \citep{curranMethodsDetectionCarelessly2016}. Accounting for these low quality responses is especially relevant in educational setting where the number of iterations an algorithm can run is often significantly smaller than those used by big tech (e.g. advertising).

Recent work in CSB has various assumptions on the adversary, the contamination, and the reward distributions. Many papers require the rewards and contamination to be bounded \citep{kapoorCorruptiontolerantBanditLearning2018, gupta2019better, lykourisStochasticBanditsRobust2018}. Others do not require boundedness, but do assume that the adversary contaminates uniformly across rewards \citep{altschulerBestArmIdentification2019}. All works make some assumption on the number of rewards for an action an adversary can contaminate. We discuss previous work more thoroughly in \Cref{sec:Related Work}.

Our work expands on these papers by allowing for a full knowledge adaptive adversary that can give unbounded contamination in any manner. However, there is a trade off when compared to work assuming bounded rewards and contamination: we require an estimate of the upper bound on the reward variance. This can often allow for simpler implementation than some algorithms that require boundedness, as we will discuss in \cref{sec:Main Results}. Our constraint on the adversary is that for some fixed $\varepsilon$, no more than $\varepsilon$ proportion of rewards for an action are contaminated. We provide a $\varepsilon$-contamination robust UCB algorithm by first proving concentration inequalities for two robust mean estimators in the $\varepsilon$-contamination context. We are able to show that the regret of our algorithm analyzed on the true reward distributions is $\bigo (\sqrt{KT \log T)}$ provided that the contamination proportion is small enough. Through simulations, we show that with a Bernoulli adversary, our algorithm outperforms algorithms designed for stochastic (UCB1) and adversarial (EXP3) bandits as well as those that have ``best of both worlds" guarantees (EXP3++ and TsallisInf) even when our constraint on the adversary is broken.

% A recent example of a bandit algorithm used in education presents an adaptive explanation improvement system \citep{williamsAXISGeneratingExplanations2016}. Participants generate a solution explanation to a math problem and the goal is to identify the best explanations in terms of improved learning outcomes for future participants. Thompson sampling is used with rewards being a rating between 1-10 provided by 150 participants who first attempt the problem and then receive a solution explanation. The results of the system showed that the bandit algorithm identified participant generated explanations that when viewed significantly improved a different participant's chance of solving future problems compared to no explanation or ``bad" explanations identified from the Thompson sampling (those with very low probability). If such a tool was implemented in a university classroom, possible contamination could be a student randomly selecting a rating, rating high when in fact it is logically incorrect, or rating low because of frustration with their own performance (e.g., there is prior work on the low quality of college student survey responses \cite{chenFindingQualityResponses2011}).

Though we are motivated by of bandit algorithms applications in education  and use this context to determine appropriate parameters in the simulations, we point out opportunities for CSB modeling to arise in other contexts as well.

\paragraph{Human feedback:} There is always a chance that human feedback is careless or inattentive, and therefore is not representative of the underlying truth related to an action. This may appear in online surveys that are used for A/B testing, or as is the case above in the explanation generation example. Adaptive surveys, such as choosing question ordering to minimize dropout rates, are also an example where the sample sizes can be small compared to other bandit deployments.

\paragraph{Click fraud:} Internet users who wish to preserve privacy can intentionally click on ads to obfuscate their true interests either manually or through browser apps. Similarly, malware can click on ads from one company to falsely indicate high interest, which can cause higher rankings in searches or more frequent use of the ad than it would otherwise merit \citep{pearceCharacterizingLargeScaleClick2014, crussellMAdFraudInvestigatingAd2014}.

% \paragraph{Transient world events:} When determining which ad to use, it is not always known ahead of time if there is a short change in the behavior of the rewards. If an ad features a brand ambassador or a location, then positive or negative actions/events of the ambassador or location can have a temporary effect on the responses to the ad.

\paragraph{Measurement errors:} If rewards are gathered through some process that may occasionally fail or be inaccurate, then the rewards may be contaminated. For example, in health apps that use activity monitors, vigorous movement of the arms may be perceived as running in place \citep{feehanAccuracyFitbitDevices2018, baiComparativeEvaluationHeart2018}.

%% What we've done
%We present a unifying framework from which to consider the effects of corruption in the stochastic multi-armed bandit problem. We introduce four basic scenarios for how corruption may affect the seen reward, and an argument for different measures of regret depending on one's belief of the problem. Finally, we provide some novel insight into something. Upper or lower bounds, or whatnot. Just one thing, anything.
%
%% Future directions
%There are many problems to tackle from here, including more complicated scenarios with changing corruption distributions and missing data.

% \textbf{Our Contributions}

% We propose using simple robust statistics and provide their concentration bounds to implement in a contamination robust UCB algorithm for sub-Gaussian true rewards. We argue for a new definition of regret under the contaminated model and give regret guarantees under this definition.

%%%%%%%%%%%%%%%%%%%%%%%%%%%%%%%%%%%%%%%%%%%%%%%%%%%%%%%%%%%%%%%%%%%%%%%%%%%%%%%%

\section{PROBLEM SETTING}
\label{sec:Problem Setting}

Here we specify our notation and present the $\varepsilon$-contaminated stochastic bandit problem. We then argue for a specific notion of regret for CSB. We compare our setting to others current in the field in \cref{sec:Related Work}.

%%%%%%%%%%%%%%%%%%%%%%%%%%%%%%%%%%%%%%%

%\ambuj{the next couple of paras mix description of our model with comparisons to other realted models. this is confusing. i think this section should focus on description of {\bf our model}. for comparisons we should just point the reader to the ``related work" section that comes later}

\paragraph{Notation} We use $[K]$ to represent $\{1, ..., K\}$ for $K \in \mathbb{R}$ to represent the number of actions and the indicator function $\mathbb{I}\{\cdot\}$ to be 1 if true and 0 otherwise. Let $N_a(t)$ be the number of times action $a$ has been chosen at time $t$ and $\textbf{x}_a(t) = \{x_{a}(1), ..., x_{a}(N_a(t))\}$ to be the vector of all observed rewards for action $a$ at time $t$. The suboptimality gap for action $a$ is $\Delta_a$ and we define $\Delta_{\min} = \min_{a \in [K]} \Delta_a$.

\subsection{{$\varepsilon$}-CONTAMINATED STOCHASTIC BANDITS}

A basic parameter in our framework is $\varepsilon$, the fraction of rewards for an action that the adversary is allowed to contaminate. Before play, the environment picks a true reward $r_a(t) \sim D_a$ from fixed distribution $D_a$ for all $a \in [K]$ and $t \in [T]$. The adversary observes these rewards and then play begins. At time $t=1, 2, ..., T$ the learner chooses an action $A_t \in [K]$ .  The adversary sees $A_t$ then chooses an observed reward $x_{A_t}(t)$ and then the learner observes only $x_{A_t}(t)$.

We present the contaminated stochastic bandits game in algorithm \ref{algo:CSB}.

\begin{algorithm}
    \DontPrintSemicolon
    \SetKwInOut{Input}{input}
    \SetKwInOut{Fix}{fix}
    \caption{Contaminated Stochastic Bandits}\label{algo:CSB}
    \Input{Number of actions $K$, time horizon $T$.}
    \Fix{$r_a(t) \ \forall a \in [K], \ t \in [T]$.}
    Adversary observes fixed rewards.\;
    \For {$t=1,...,T$}{
        Learner picks action $A_t \in [K]$.\;
        Adversary observes $A_t$ and chooses $x_{A_t}(t)$.\;
        Learner observes $x_{A_t}(t)$.\;
        }
\end{algorithm}

We allow the adversary to corrupt in any fashion as long as for every time $t$ there is no more than an $\varepsilon$-fraction of contaminated rewards for any action. That is, we constrain the adversary such that,
$$\forall a \in [K], \ \forall t \in [T], \  \sum_{i=1}^{N_a(t)} \mathbb{I}\{r_{a}(i) \neq x_{a}(i)\} \leq \varepsilon \cdot N_a(t).$$
We allow the adversary to give unbounded contamination that can be chosen with full knowledge of the learner's history as well as current and future rewards. This setting allows the adversary to act differently across actions and places no constraints on the contamination itself, but rather the rate of contamination.

\subsection{NOTION OF REGRET}\label{notionsOfRegret}

%\ambuj{IMO, we are introducing our notion of regret too late in the paper. it should be introduced as part of our problem setting so that our discussion on adversarial bandits (along with the regret measure they minimizer) makes more sense to the reader. the main results section should focus on results, not on the problem setting and criterion of assessing good algorithms}

A traditional goal in bandit learning is to minimize the observed cumulative regret gained over the total number of plays $T$. Because the adversary in this model can affect the observed cumulative regret, we argue to instead use a notion of regret that considers only the underlying true rewards.
%This is the typical notion of regret as if we can observe the true, uncontaminated rewards.
We call this uncontaminated regret and give the definition below for any time $T$ and policy $\pi$ in terms of the true rewards $r$,
\begin{align}
\label{eq:true_regret}
  \bar{R}_T(\pi) = \underset{a \in [K]}{\max}
    \EX \bigg[ \sum_{t=1}^T r_a(t) - \sum_{t=1}^T r_{A_t}(t)\bigg].
\end{align}
This definition \cref{eq:true_regret}  is first mentioned in \citet{kapoorCorruptiontolerantBanditLearning2018} along with another notion of regret that compares the sum of the observed (possibly contaminated) rewards to the sum of optimal, uncontaminated rewards,
\begin{align}
\label{eq:t_o_regret}
  \bar{R}_T(\pi) = \underset{a \in [K]}{\max}
    \EX \bigg[ \sum_{t=1}^T r_a(t) - \sum_{t=1}^T x_{A_t}(t)\bigg].
\end{align}
We argue that \cref{eq:t_o_regret} gives little information about the performance of an algorithm. This notion of regret can be negative, and with no bounds on the contamination it can be arbitrarily small and potentially meaningless. We believe that any regret that compares a true component to an observed (possibly contaminated) component is not a useful measure of performance in CSB as it is unclear what regret an optimal strategy should produce.

%%%%%%%%%%%%%%%%%%%%%%%%%%%%%%%%%%%%%%%%%%%%%%%%%%%%%%%%%%%%%%%%%%%%%%%%%%%%%%%%

\section{RELATED WORK}
\label{sec:Related Work}

We start by briefly addressing why adversarial and ``best of both world" algorithms are not optimized for CSB. We then cover relevant work in robust statistics, followed by current work in robust bandits and how our model differs and relates.
%\ambuj{it's not forbidden but often considered bad stylistically to have no intervening text between section heading and subsection heading}

%%%%%%%%%%%%%%%%%%%%%%%%%%%%%%%%%%%%%%%

\subsection{ADVERSARIAL BANDITS}
Adversarial bandits with an oblivious environment allows the adversary to first look at the learners policy and then choose all rewards before the game begins. If the learner chooses a deterministic policy, the adversary can choose rewards such that the learner cannot achieve sublinear worst-case regret \citep{lattimoreBanditAlgorithms2018}. Algorithms such as EXP3 \citep{auerNonstochasticMultiarmedBandit2002} are thus randomized, but their regret is analysed with respect to the best fixed action where ``best" is defined using the {\em observed} rewards. There are no theoretical guarantees with respect to the uncontaminated regret, so it is not immediately clear how they will perform in a CSB problem. We remark that adversarial analysis assumes uniformly bounded observed rewards whereas we allow observed rewards to be unbounded. Additionally, the general adversarial framework does not take advantage of the structure present in CSB, namely that the adversary can only corrupt a small fraction of rewards, so it is likely that performance improvements can be made.

\subsection{BEST OF BOTH WORLDS}
A developing line of work is algorithms that enjoy ``best of both worlds" guarantees. That is, they perform well in both stochastic and adversarial environments without knowing a priori which environment they will face. Early work in this area   \citep{auerAlgorithmNearlyOptimal2016, bubeckBestBothWorlds2012} started by assuming a stochastic environment and implementing some method to detect a failure of the i.i.d.\ assumption on rewards, at which point the algorithm switches to an algorithm for the adversarial environment for the remainder of iterations. Further work implements algorithms that can handle an environment that is some mixture of stochastic and adversarial, as in EXP3++ and TsallisInf \citep{seldinOnePracticalAlgorithm2014, zimmertAnOptimal2019}.

While these algorithms are aimed well for a stochastic environment with some adversarial rewards, they differ from contamination robust algorithms in that all observed rewards are thought to be informative. Their uncontaminated regret has not been analysed and therefore there are no guarantees in the CSB setting.

%%%%%%%%%%%%%%%%%%%%%%%%%%%%%%%%%%%%%%%
\subsection{CONTAMINATION ROBUST STATISTICS}

The $\varepsilon$-contamination model we consider is closely related to the one introduced by Huber in 1964 \citep{huberRobustEstimationLocation1964}. Their goal was to estimate the mean of a Gaussian mixture model where $\varepsilon$ fraction of the sample was not sampled from the main Gaussian component. There has been a recent increase of work using this model, especially in extensions to the high-dimensional case
(\cite{diakonikolasRobustEstimatorsHigh2019}, \cite{kothariOutlierrobustMomentestimationSumofsquares2017}, \cite{laiAgnosticEstimationMean2016}, \cite{liuHighDimensionalRobust2019}). These works often keep the assumption of a Gaussian mixture component, though there has been expanding work with non-Gaussian models as well.

%Some researchers \cite{laiAgnosticEstimationMean2016} have suggested using the median as an estimate for the mean in the Gaussian case with adversarial $\varepsilon$-fraction contamination. They provide concentration bounds under $\varepsilon$-contamination for this estimator, making it immediately applicable in a UCB algorithm, as we mention below.

%%%%%%%%%%%%%%%%%%%%%%%%%%%%%%%%%%%%%%%
\subsection{CONTAMINATION ROBUST BANDITS}

%\ambuj{again, a gentle reminder to not use citations in square brackets are part of sentences. if you really like using citations in your sentences switch to a bib style that offers a {\tt $\backslash$citep} command for parenthetical citations as well as a  {\tt $\backslash$citet} command for citations meant to be part of the text}
Some of the first work in CSB started by assuming both rewards and contamination were bounded \citep{lykourisStochasticBanditsRobust2018,gupta2019better}. These works assume an adversary that can contaminate at any time step, but that is constrained in the cumulative contamination. They bound the cumulative max (over actions) absolute difference of the contaminated reward, $x$, to the true reward, $r$, $ \sum_t \max_a |r_a(t)-x_a(t)| \leq C$. \citet{lykourisStochasticBanditsRobust2018} provides a layered UCB-type active arm elimination algorithm. \citet{gupta2019better} expands on this work to provide an algorithm similar to active arm elimination in spirit, but which never completely eliminates an action, and which has better regret guarantees.

Recent work in implementing a robust UCB replaces the empirical mean with the empirical median, and gives guarantees for the uncontaminated regret with Gaussian rewards \citep{kapoorCorruptiontolerantBanditLearning2018}. They consider an adaptive adversary but require the contamination to be bounded, though the bound need not be known. They cite work that can expand their robust UCB to distributions with bounded fourth moments by using the agnostic mean \citep{laiAgnosticEstimationMean2016}, though give no uncontaminated regret guarantees. In one dimension, the agnostic mean takes the mean of the smallest interval containing $(1-\alpha)$ fraction of points. This estimator is also known as the $\alpha$-shorth mean. Our work expands on this model by allowing for unbounded contamination and analysing the uncontaminated regret for sub-Gaussian rewards when implementing a UCB algorithm with the $\alpha$-shorth mean.

CSB has also been analysed in the best arm identification problem \citep{altschulerBestArmIdentification2019}. Using a Bernoulli adversary that contaminates any reward with probability $\varepsilon$, \citet{altschulerBestArmIdentification2019} consider three adversaries of increasing power, from the oblivious adversary, which does not know the player's history nor the current action or reward, to a malicious adversary, which can contaminate knowing the player's history and the current action and reward. They give analysis of the probability of best arm selection and sample complexity of an active arm elimination algorithm. While their performance measure is different than ours, we generalize their context to allow an adversary to contaminate in any fashion.

There is also work that explores the impact of an adaptive adversarial contamination on $\varepsilon$-greedy and UCB algorithms \citep{junAdversarialAttacksStochastic2018}. They give a thorough analysis with both theoretical guarantees and simulations of the effects an adversary can have on these two algorithms when the adversary does not know the optimal action but is otherwise fully adaptive. They show these standard algorithms are susceptible to contamination. Similar work looks at contamination in contextual bandits with a non-adaptive adversary \citep{ma2019data}.

%%%%%%%%%%%%%%%%%%%%%%%%%%%%%%%%%%%%%%%%%%%%%%%%%%%%%%%%%%%%%%%%%%%%%%%%%%%%%%%%

\section{MAIN RESULTS}
\label{sec:Main Results}
We present concentration bounds for both the $\alpha$-shorth and $\alpha$-trimmed mean estimators in the $\varepsilon$-contamination context for sub-Gaussian random variables.

Our contribution to the CSB problem is in providing a contamination robust UCB algorithm that is simple to implement and has theoretical regret guarantees close to those of UCB algorithms in the uncontaminated setting.

\subsection{CONTAMINATION ROBUST MEAN ESTIMATORS}
The estimators we analyse have been in use for many decades as robust statistics. Our contribution is to analyze them within our $\varepsilon$-contamination model with sub-Gaussian samples and provide simple {\em finite-sample concentration inequalities} for ease of use in UCB-type algorithms.

\subsubsection{Trimmed Mean}

Our first estimator suggested for use in the contaminated model is the $\alpha$-trimmed mean \citep{liuHighDimensionalRobust2019}.
\paragraph{$\alpha$-trimmed mean} Trim the smallest and largest $\alpha$-fraction of points from the sample and calculate the mean of the remaining points. This estimator uses $1-2\alpha$ fraction of sample points.

\begin{algorithm}
    \DontPrintSemicolon
    \SetKwInOut{Input}{input}
    \SetKwInOut{Output}{output}
    \caption{$\alpha$-Trimmed Mean}\label{algo:tMean}
    \Input{$X_n = (x_1, ..., x_n)$, $\alpha$}
    \Output{$\alpha$-trimmed mean}
    $(x_{(1)}, ..., x_{(n)})$ = sorted $X_n \ s.t. \ x_{(i)} \leq x_{(i+1)}$\;
    cut = $\ceil{\alpha*n}$\\
    \Return{mean($x_{\text{(cut)}}, ..., x_{\text{(n-cut)}}$)}
\end{algorithm}

The intuition being if the contamination is large, then it will be removed from the sample. If it is small, it should have little affect on the mean estimate. Next we provide the concentration inequality for the $\alpha$-trimmed mean.
\begin{restatable}[Trimmed mean concentration]{thm}{trMeanConcenG}
  \label{thm:trMeanConcenG}
   Let $G$ be the set of points $x_1, ... x_n \in \mathbb{R}$ that are drawn from a $\sigma$-sub-Gaussian distribution with mean $\mu$. Let $S_n$ be a sample where an $\varepsilon$-fraction of these points are contaminated by an adversary. For $\varepsilon \leq \alpha < 1/2$, $t \geq n$ we have,
   \begin{align*}
     |\text{trMean}_{\alpha}&(S_n) - \mu| \leq\\
     & \frac{\sigma}{(1-2\alpha)} \bigg(
       \sqrt{\frac{4}{n} \log(t)}
       + 4 \alpha\sqrt{6\log(t)} \bigg)
   \end{align*}
   with probability at least $1 - \frac{4}{t^2}$.
\end{restatable}

%\begin{proof}
%\label{pf:trMean}
 Proof follows from \citet{liuHighDimensionalRobust2019} and can be found in the appendix.

\subsubsection{Shorth Mean}
The agnostic mean from \citet{laiAgnosticEstimationMean2016}, which we use the more common term $\alpha$-shorth mean for, can be considered a variation of the trimmed mean.
\paragraph{$\alpha$-shorth mean} Take the mean of the shortest interval that removes the smallest $\delta_1$ and largest $\delta_2$ fraction of points such that $\delta_1 + \delta_2 = \alpha$, where $\delta_1, \ \delta_2$ are chosen to minimize the interval length of remaining points. Uses $1-\alpha$ fraction of sample points.

The $\alpha$-shorth mean is less computationally efficient than the trimmed mean, but may be a better mean estimator when the contaminated points are not large outliers and are skewed in one direction. Intuitively this is because the $\alpha$-shorth mean can trim off contamination that would require removing most of the sample with the trimmed mean. Next we provide the concentration inequality for the $\alpha$-shorth mean.

\begin{algorithm}
 \DontPrintSemicolon
 \SetKwInOut{Input}{input}
 \SetKwInOut{Output}{output}
 \caption{$\alpha$-Shorth Mean}\label{algo:sMean}
 \Input{$X_n = (x_1, ..., x_n)$, $\alpha$}
 \Output{A mean estimate for the distribution of $X$}
    $(x_{(1)}, ..., x_{(n)})$ = sorted $X_n \ s.t. \ x_{(i)} \leq x_{(i+1)}$\;
    $n_{\alpha}$ = $\floor{(1-\alpha)*n}$\;
 $\mathcal{I} \in \argmin_{k} \{x_{(k+n_{\alpha})} - x_{(k)}\}$\;
 Choose uniformly at random from set $\mathcal{I}$ if there is more than one starting index with the smallest interval length\;
 \Return{sMean$(X)\gets$ mean$(x_{(\mathcal{I})}, ..., x_{(\mathcal{I}+n_{\alpha})})$}
\end{algorithm}

\begin{restatable}[$\alpha$-shorth mean concentration]{thm}{sMeanConcenG}
  \label{thm:sMeanConcenG}
   Let $G_n$ be the set of points $x_1, ... x_n \in \mathbb{R}$ that are drawn from a $\sigma$-sub-Gaussian distribution with mean $\mu$. Let $S_n$ be a sample where an $\varepsilon$-fraction of these points are contaminated by an adversary. For $\varepsilon \leq \alpha < 1/3$, $t \geq n$, we have,
   \begin{align*}
     |\text{sMean}_{\alpha}&(S_n) - \mu|
       \leq \\
      &  \frac{\sigma}{1-2\alpha} \sqrt{\frac{4}{n} \log t}
         + \frac{(6\alpha-8\alpha^2) \sigma}{(1-2\alpha)(1-\alpha)} \sqrt{6\log t}
   \end{align*}
   with probability at least $1 - \frac{4}{t^2}$.
\end{restatable}

\begin{proof}[Proof sketch]
\label{pf:sMean}
Without loss of generality assume $\mu = 0$ for the underlying true distribution. Let $\tilde{G} \subset G_n$ represent the points which are not contaminated and $C \subset G_n$ represent the contaminated points. Then our sample can be represented by the union $S_n = \tilde{G} \cup C$

Let $J$ be the interval that contains the shortest $1-\alpha$ fraction of $S_n$, $I$ be the interval that contains $\tilde{G}$ (i.e. the remaining good points after contamination), and $T$ be the interval that contains the points of $S_n$ after trimming the $\alpha$ largest and smallest fraction of points. Use $|I|$ to denote the length of interval $I$. It must be that $I \cap J \neq \emptyset$ because otherwise the points in $I\cup J$ would contain $2 - 2\alpha > 1$ fraction of $S_n$. Let $c$ be a point in $I \cap J$ and $x$ be a point in $J$. Recall that trMean$_{\alpha}(S_n)$ is the trimmed mean of the contaminated sample $S_n$. Then we have,
 \begin{align*}
    |x| &\leq |x-c| + |c- \text{trMean}_{\alpha}(S_n)| + |\text{trMean}_{\alpha}(S_n)|\\
        &\leq |J| + |I| + |\text{trMean}_{\alpha}(S_n)|\\
        &\leq 2|I| + |\text{trMean}_{\alpha}(S_n)|
 \end{align*}
 The second step comes from $x$ and $c$ both being in $J$ and because $I \supseteq T$. The third step comes from $|J| \leq |I|$.

To bound the length of $I$ we have,
 \begin{align*}
    |I| &\leq 2 \max_{x \in G_n} |x| \ \ \text{w.p. at least } 1-\delta_2.
 \end{align*}
Finally, since
\[
 |\text{trMean}_{\alpha}(S_n)| \leq \frac{1}{(1-2\alpha)}(|\bar{x}_{G_n}|
     + 4\alpha \max_{x \in G_n} |x| )
\]
with probability at least $1-\delta_1 - \delta_2$, we get that for $x \in J$,
\begin{align*}
 |x| &\leq 4 \max_{i \in [n]} |x_i|
            +\frac{1}{(1-2\alpha)}(|\bar{x}_{G_n}|
            + 4\alpha \max_{x \in G_n} |x| )\\
     &= \frac{|\bar{x}_{G_n}|}{1-2\alpha}
            + \Big(4 + \frac{4\alpha}{1-2\alpha}\Big)\max_{x \in G_n} |x|.
\end{align*}
Now that we have a bound on the contaminated points in $J$, our analysis follows similarly as the trimmed mean by bounding $A_1, A_2, A_3$ as defined below.
\begin{align*}
 &|\text{sMean}_{\alpha}(S_n)| \\
        &\leq \frac{1}{(1-\alpha)n} \bigg(
         \bigg| \underbrace{\sum_{x \in \tilde{G}} x}_{A_1} \bigg|
        + \bigg| \underbrace{\sum_{x \in \tilde{G} \cap \neg J} x}_{A_2} \bigg|
        + \bigg| \underbrace{\sum_{x \in C \cap J} x}_{A_3} \bigg|
     \bigg)
\end{align*}
\end{proof}

The full proof is contained in the appendix and follows a similar approach as for the trimmed mean.

Our methods ensured that the first term in each concentration bound is the same, giving them similar regret guarantees when implemented in a UCB algorithm. We emphasize that the $\alpha$-shorth mean uses $1-\alpha$ fraction of a sample while the $\alpha$-trimmed mean uses $1-2\alpha$ fraction of a sample. We remark that if there is no contamination and $\alpha=0$ then our inequalities reduce to the standard concentration inequality for the empirical mean of samples drawn from a sub-Gaussian distribution.

\subsection{CONTAMINATION ROBUST UCB}

We present the contamination robust-UCB (crUCB) algorithm for $\varepsilon$-CSB with sub-Gaussian rewards.

\begin{algorithm}
    \DontPrintSemicolon
    \SetKwInOut{Input}{input}
    \caption{crUCB}\label{algo:rUCBG}
    \Input{number of actions $K$,
            time horizon $T$,
            upper bound on fraction contamination $\alpha$,
            upper bound on sub-Gaussian constant $\sigma_0$,
            mean estimate function ($\alpha$ trimmed or shorth mean) $f$.}
    \For {$t \leq K$}{
        Pick action $a$ when $t = a$.
    }
    \For {$t > K$}{
      \For {$a \in [K]$ compute}{
        $f(\textbf{x}_a(t)) \gets$ mean estimate of rewards.\;
        $N_a(t) \gets$ number of times action has been played.\;
      }
      Pick action $A_t =
      \text{argmax}_{a \in [K]} f(\textbf{x}_a(t))
                + \frac{\sigma_0}{(1-2\alpha)}\bigg(
                \sqrt{4\frac{\log(t)}{N_a(t)} }
                %+ 4 \alpha\sqrt{6\log(t)}
                \bigg)$.\;
      Observe reward $x_{A_t}(t)$.\;
    }
\end{algorithm}

We provide uncontaminated regret guarantees for crUCB below for both the $\alpha$-trimmed and the $\alpha$-shorth mean.

\begin{restatable}[$\alpha$-trimmed mean crUCB uncontaminated regret]{thm}{trMeanRegretG}
  \label{thm:trMeanRegretG}
  Let $K > 1$ and $T \geq K-1$. Then with algorithm \ref{algo:rUCBG} with the $\alpha$-trimmed mean, $\sigma$-sub-Gaussian reward distributions with $\sigma_a \leq \sigma_0$, and contamination rate $\varepsilon \leq \alpha \leq \frac{\Delta_{min}}{4(\Delta_{min} + 4\sigma_0 \sqrt{6\log T})}$, we have the uncontaminated regret bound,

\[
  \bar{R}(UCB) \leq 8\sigma_0 \sqrt{KT\log T} + \sum 15\Delta_a.
\]
\end{restatable}

\begin{restatable}[$\alpha$-trimmed mean crUCB uncontaminated regret bounded rewards]{cor}{trMeanb}
    \label{cor:tr}
    If the rewards are bounded by $b$, and have contamination rate $\varepsilon \leq \alpha \leq \frac{\Delta_{\min}}{4(\Delta_{\min}+4b)}$, then
    \[
  \bar{R}_T \leq 8\sigma_0\sqrt{KT\log(T)} + \sum 15\Delta_a.
 \]
\end{restatable}

%proof in \cref{pf:trMeanRegretG}

\begin{restatable}[$\alpha$-shorth mean crUCB uncontaminated regret]{thm}{sMeanRegretG}
  \label{thm:sMeanRegretG}
  Let $K > 1$ and $T \geq K-1$. Then with algorithm \ref{algo:rUCBG} with the $\alpha$-shorth mean, sub-Gaussian reward distributions with $\sigma_a \leq \sigma_0$, and contamination rate $\varepsilon \leq \alpha \leq \frac{\Delta_{min}}{4(\Delta_{min}+9\sigma_0\sqrt{6\log T})}$, we have the uncontaminated regret bound,

\[
  \bar{R}(UCB) \leq 8\sigma_0 \sqrt{KT\log T} + \sum 15\Delta_a.
\]
\end{restatable}

\begin{restatable}[$\alpha$-shorth mean crUCB uncontaminated regret bounded rewards]{cor}{shMeanb}
    \label{cor:sh}
    If the rewards are bounded by $b$, and have contamination rate $\varepsilon \leq \alpha \leq \frac{\Delta_{\min}}{4(\Delta_{\min}+9b)}$, then
    \[
  \bar{R}_T \leq 8\sigma_0\sqrt{KT\log(T)} + \sum 15\Delta_a.
 \]
\end{restatable}

Proofs for \cref{thm:trMeanRegretG} and \ref{thm:sMeanRegretG} and their corollaries follow standard analysis and are provided in the appendix.% \cref{pf:sMeanRegretG}.

From \cref{thm:trMeanRegretG} and \ref{thm:sMeanRegretG} we get that crUCB has the same order of regret in the CSB setting as UCB1 has in the standard sMAB setting. The constraint on the magnitude of $\varepsilon$ is quite strong, but we show in \cref{sec:Simulations} that they can be broken and still obtain good empirical performance.
\paragraph{Remark} Our bounds above do not allow $\varepsilon$ to be too big relative to the minimum suboptimality gap $\Delta_{\min}$. This is natural: if $\varepsilon > \Delta_{\min}$ then no algorithm can get sublinear regret since distinguishing between the top two actions is statistically impossible even with infinite samples. We give a simple example in Appendix B.
%~\ref{sec:deltamin}.
Furthermore, it is possible to derive a regret bound\footnote{The $\tilde{O}(\cdot)$ notation hides constants and logarithmic terms. See Appendix B for details.}
%~\ref{sec:deltamin} for details.}
of $\tilde{O}(\sigma_0 \sqrt{KT} + \tfrac{\alpha\sigma_0}{1-4 \alpha} T)$ for any choice of $\alpha$ such that $\varepsilon \le \alpha < 1/4$. The linear term in regret (which is unavoidable for large $\varepsilon$) may be acceptable if the corruption proportion is not very large.
%\ambuj{a little bit of discussion is needed to explain what the theorems are saying. we need to say something about the magnitude of the allowed contamination rate. also, we need to point out that if contamination rate is higher than $\Delta_{min}$, no algorithm can get sublinear regret since distinguishing between the top two arms can then becomes statistically impossible even with infinite samples}

%%%%%%%%%%%%%%%%%%%%%%%%%%%%%%%%%%%%%%%%%%%%%%%%%%%%%%%%%%%%%%%%%%%%%%%%%%%%%%%%

\section{SIMULATIONS}
\label{sec:Simulations}
We compare our crUCB algorithms using the trimmed mean (tUCB) and shorth mean (sUCB) against a standard stochastic algorithm (UCB1, \cite{auerFinitetimeAnalysisMultiarmed2002}), a standard adversarial algorithm (EXP3, \cite{auerNonstochasticMultiarmedBandit2002}), two ``best of both worlds'' algorithms (EXP3++, \cite{seldinImprovedParametrizationAnalysis2017}, 0.5-TsallisInf, \cite{zimmertAnOptimal2019}), and another contamination robust algorithm (RUCB-MAB, \cite{kapoorCorruptiontolerantBanditLearning2018}). Each trial has five actions ($K=5$), is run for 1000 iterations ($T=1000$), for $\varepsilon \in \{0.05, 0.1\}$. For sUCB and tUCB, we set $\alpha = \varepsilon$ and $\sigma_0 = \sigma$. The plots are average results over 10 trials with error bars showing the standard deviation.

Our choice of $T$ comes from our motivation to apply contaminated bandits in education, where the sample sizes are often much smaller than for example in advertising. While $T=1000$ would be considered a large university class, it still allows one to visually see regret for smaller iterations and see how performance stabilizes.  We similarly chose number $K$ of arms and proportion contamination $\varepsilon$ to be in a realistic range for the application we have in mind. All algorithms use recommended parameter settings given within their respective papers. %We use $\Delta \in \{.1, .3\} \cdot$range-of-rewards
%\ambuj{is it really realistic to assume $\Delta$ is in the range 1-3 range when the rewards are in the 0-10 range. that looks to me like a pretty big gaps. in education applications, isn't it unlikely that the best action is so much better than the second best?}

%\ambuj{also need to explain why you have a single $\Delta$ instead of multiple $\Delta_a$, namely that you decided to keep all suboptimal arms equallu suboptimal}

\paragraph{Rewards and gaps} We chose the reward distribution to be binomial(n=10) to simulate likert scale and because this distribution has bounded rewards and is not symmetric for large $p$. For the optimal action, $p = .9$ and for suboptimal actions $p=.8$, thus the suboptimality gap is $\Delta =1$. All non-optimal actions have the same true distribution.

\paragraph{Adversaries} We focus on a Bernoulli adversary which gives a contaminated reward at every time step with probability $\varepsilon$. We also implement a cluster adversary which contaminates at the beginning of play to show the weakness of algorithms to this type of attack.

\paragraph{Contamination} We use a random malicious contamination scheme which chooses a contaminated reward uniformly from ranges that increase suboptimal action means and decrease the optimal action's mean.

\paragraph{Performance measurement} We plot the average regret over 10 trials for 1000 iterations.

We recommend to view the plots on a color screen.

\begin{figure}[h!]
\centering
  \begin{subfigure}{.4\textwidth}
    \centering
    \includegraphics[width=\textwidth]{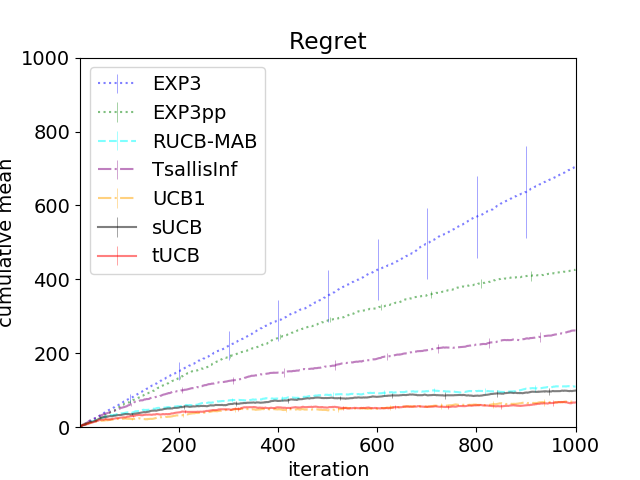}
    \caption{$\varepsilon=0$}
    \label{subfig:0}
  \end{subfigure}%

  \begin{subfigure}{.4\textwidth}
    \centering
    \includegraphics[width=\textwidth]{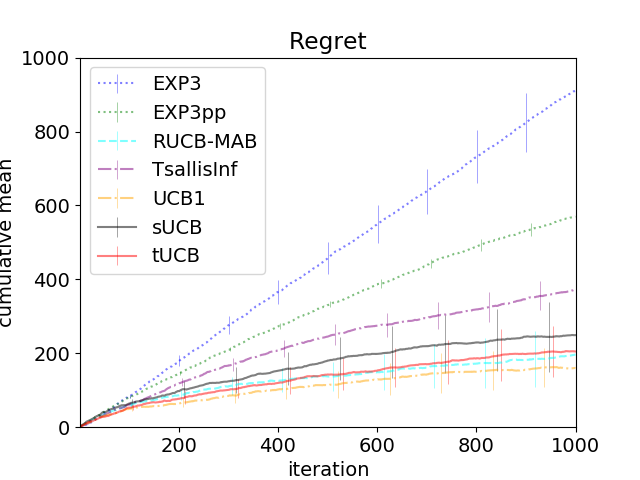}
    \caption{$\varepsilon=0.05$}
    \label{subfig:.05}
  \end{subfigure}

  \begin{subfigure}{.4\textwidth}
    \centering
    \includegraphics[width=\textwidth]{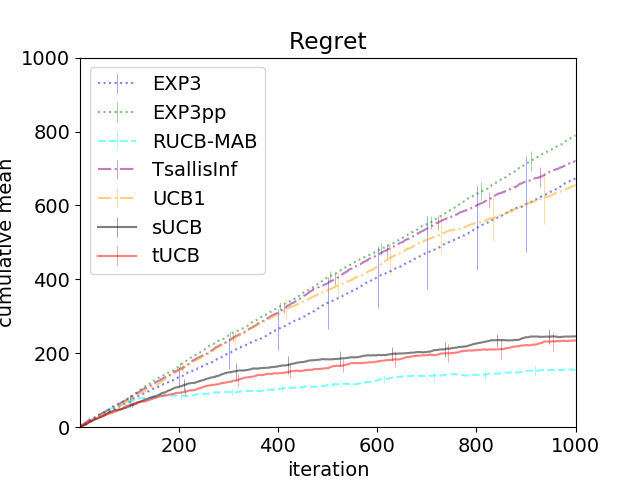}
    \caption{$\varepsilon=0.1$}
    \label{subfig:.1}
  \end{subfigure}%
  \caption{Binomial Rewards With Varying Proportion Of Contamination}
  \label{fig:binomial_true}
\end{figure}

\begin{figure}[h!]
\centering
  \begin{subfigure}{.4\textwidth}
    \centering
    \includegraphics[width=\textwidth]{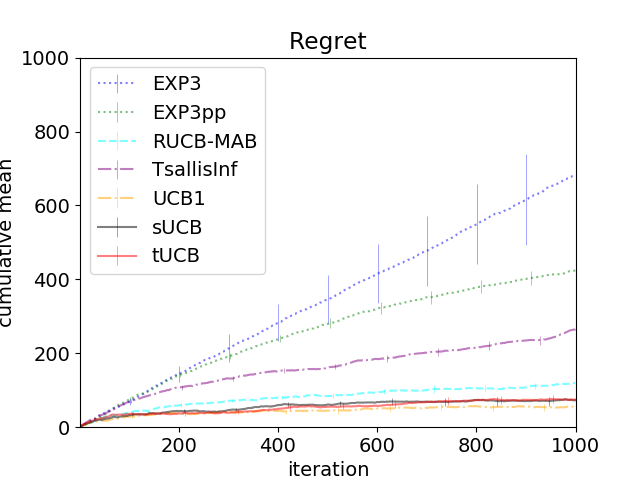}
    \caption{$\alpha=0$}
    \label{subfig:binomial_miss_0}
  \end{subfigure}%

  \begin{subfigure}{.4\textwidth}
    \centering
    \includegraphics[width=\textwidth]{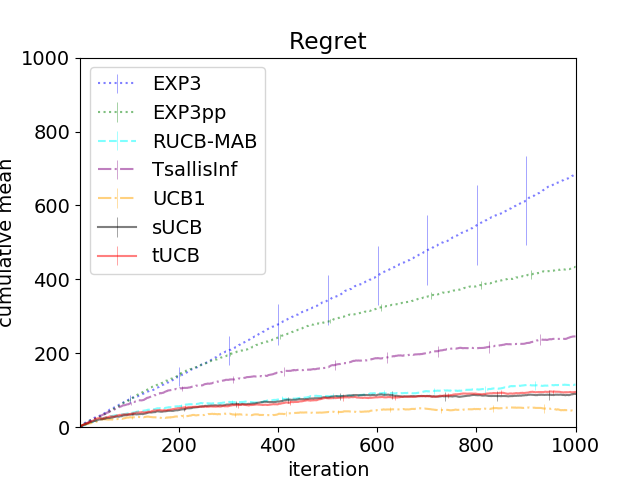}
    \caption{$\alpha=0.05$}
    \label{subfig:binomial_miss_0.05}
  \end{subfigure}%

  \begin{subfigure}{.4\textwidth}
    \centering
    \includegraphics[width=\textwidth]{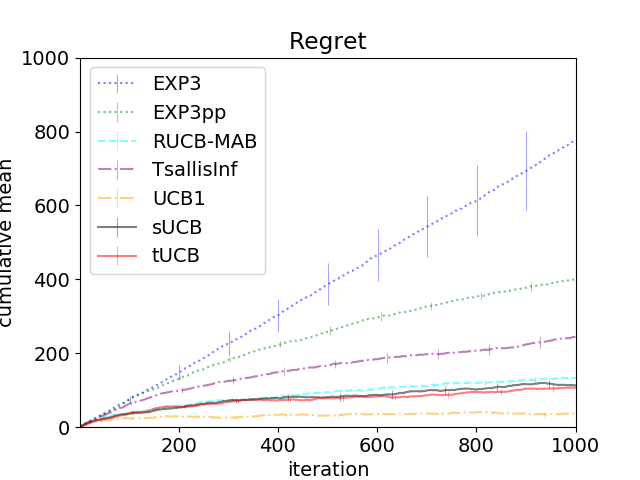}
    \caption{$\alpha=0.1$}
    \label{subfig:binomial_miss_0.1}
  \end{subfigure}%
  %\begin{subfigure}{.25\textwidth}
  %  \centering
  %  \includegraphics[width=\textwidth]{{misspec%ified/Regretbinomial_NoAdversary_random_03}.png}
  %  \caption{$\alpha=0.3$}
  %  \label{subfig:binomial_miss_0.3}
  %\end{subfigure}%
  \caption{Misspecified $\alpha$ For $\varepsilon=0$.}
  \label{fig:misspecified}
\end{figure}

\begin{figure}[t]
\centering
  \begin{subfigure}{.4\textwidth}
    \centering
    \includegraphics[width=\textwidth]{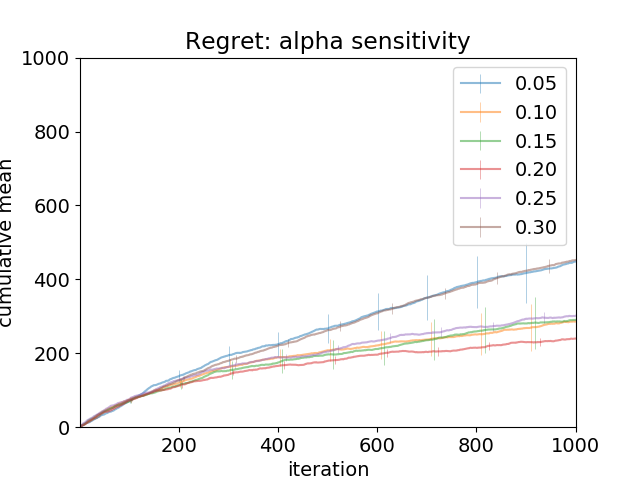}
    \caption{sUCB}
    \label{subfig:binomial_sens_sUCB}
  \end{subfigure}%

  \begin{subfigure}{.4\textwidth}
    \centering
    \includegraphics[width=\textwidth]{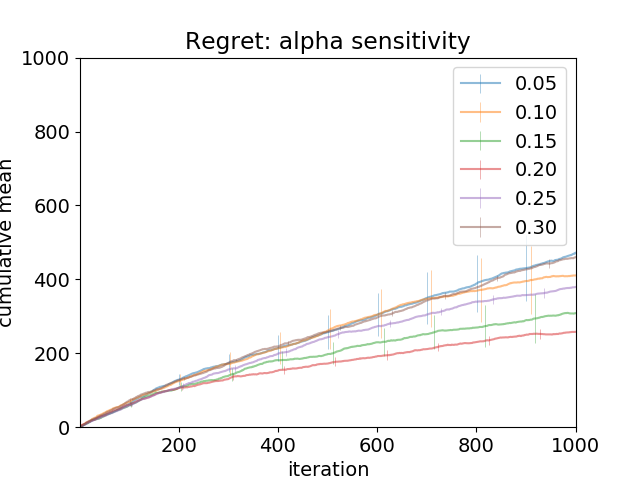}
    \caption{tUCB}
    \label{subfig:binomial_sens_tUCB}
  \end{subfigure}%
\caption{Regret Sensitivity For Various
$\alpha$.}
  \label{fig:sensitivity alpha}
\end{figure}

\begin{figure}[t]
\centering
  \begin{subfigure}{.4\textwidth}
    \centering
    \includegraphics[width=\textwidth]{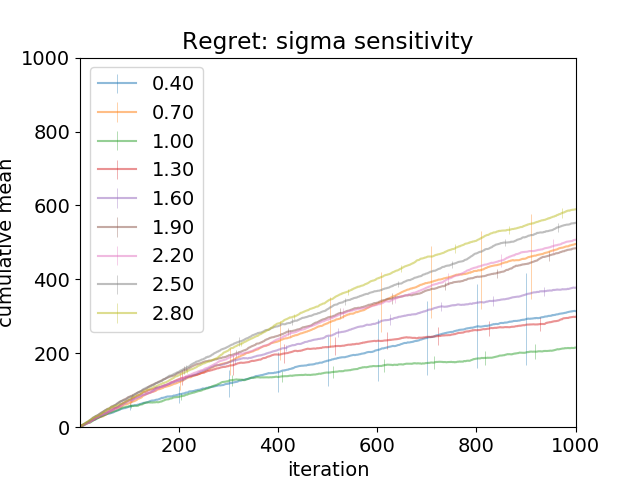}
    \caption{sUCB}
    \label{subfig:binomial_sig_sUCB}
  \end{subfigure}%

  \begin{subfigure}{.4\textwidth}
    \centering
    \includegraphics[width=\textwidth]{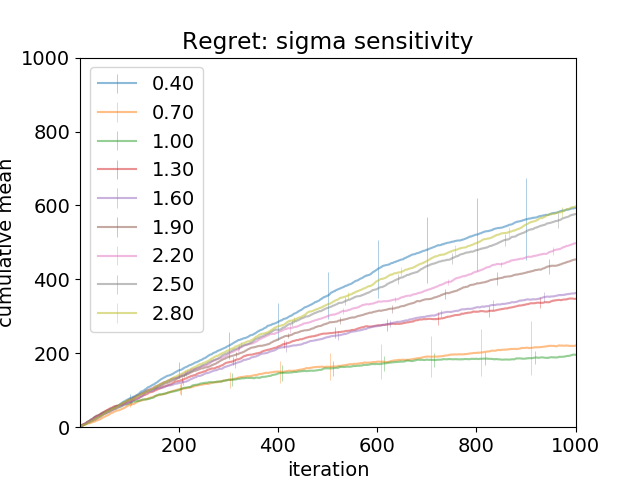}
    \caption{tUCB}
    \label{binomial_sig_tUCB}
  \end{subfigure}
  \caption{Regret Sensitivity For Various $\sigma$.}
  \label{fig:sensitivity sig}
\end{figure}

In \Cref{subfig:0} we see that the adversarial and best of both worlds algorithms, EXP3, EXP3++, and TsallisInf, perform poorly in the purely stochastic setting compared to the UCB type algorithms. In \Cref{fig:binomial_true}, we see the best of these, TsallisInf, starts to degrade as the proportion of contamination increases while the robust UCB algorithms are only slightly affected. These simulations show a clear performance benefit to using algorithms that specifically account for contaminated rewards.

\Cref{fig:sensitivity alpha} and \Cref{fig:sensitivity sig} shows that for both sUCB and tUCB, the choice of $\alpha$ is much less sensitive than choice of $\sigma$. Over estimating or slightly underestimating $\alpha$ does not degrade performance significantly. Underestimating $\sigma$ can give a significant boost to performance while over estimating can degrade it. This is consistent with the performance of UCB algorithms in practice, which often scale the exploration term to improve empirical performance \citep{liu2014trading}.

To look at the impact of using a contamination robust algorithm when there is no contamination, we plotted various $\alpha$ values when $\varepsilon=0$, shown in \Cref{fig:misspecified}. Assuming small amounts of contamination when there is none only has a small impact on performance, suggesting it is permissible to use contamination robust methods when there is uncertainty of contamination. Similarly, small $K$ and large $\Delta$ can render bounded contamination impotent and would not require algorithms that account for it.

We have included RUCB-MAB in our simulations because it is simple to implement and can perform similarly well to our algorithms. We note it currently has guarantees only for Gaussian rewards \citep{kapoorCorruptiontolerantBanditLearning2018}.

\begin{figure}[h]
\centering
%begin{subfigure}{.4\textwidth}
%    \centering
%    \includegraphics[width=\textwidth]{{Regretbinomial_Cluster%_random_005}.png}
%    \caption{$\varepsilon=.05$}
%    \label{subfig:cluster005}
%  \end{subfigure}%

  \begin{subfigure}{.4\textwidth}
    \centering
    \includegraphics[width=\textwidth]{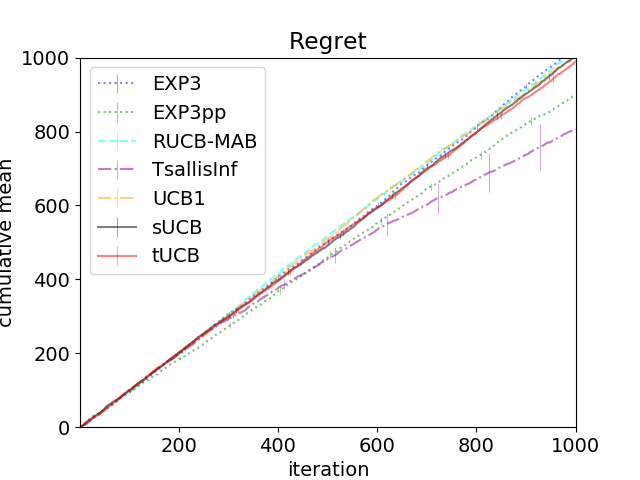}
    \caption{$\varepsilon=.1$}
    \label{subfig:cluster01}
  \end{subfigure}%
\caption{Front Cluster Attack}
  \label{fig:cluster}
\end{figure}

\Cref{fig:cluster} shows the poor performance of all algorithms when the first $\varepsilon$ rewards are contaminated. TsallisInf and EXP3++ show some recovery, but it is clear this type of adversary is harmful. This remains an open problem for scenarios with small $T$.

We also considered including the BARBAR algorithm \citep{gupta2019better} whose epoch scheme is the only algorithm we know that accounts for the front cluster attack. We chose against this as for our setting of $T=1000$ the BARBAR algorithm only has one epoch, and thus does not make any updates to the estimated gaps, resulting in pure random exploration.

%%%%%%%%%%%%%%%%%%%%%%%%%%%%%%%%%%%%%%%%%%%%%%%%%%%%%%%%%%%%%%%%%%%%%%%%%%%%%%%%
\section{DISCUSSION}

We have presented two variants of an $\varepsilon$-contamination robust UCB algorithm to handle uninformative or malicious rewards in the stochastic bandit setting. As the main contribution, we proved concentration inequalities for the $\alpha$-trimmed and $\alpha$-shorth mean in the $\varepsilon$-contamination setting with sub-Gaussian samples and guarantees on the uncontaminated regret of the crUCB algorithms. The regret guarantees are similar to those in the uncontaminated sMAB setting.

We have shown through simulation that these algorithms can outperform ``best of both worlds'' algorithms and those for stochastic or adversarial environments when using a small number of iterations and $\varepsilon$ chosen to be reasonable when implementing bandits in education.

We highlight that our algorithms are simple to implement. In practice, it is often easy to find upper bounds on the parameters which are robust to underestimation. Our algorithms are numerically stable and have clear intuition to their actions.

A weak point of these algorithms is they require knowledge of $\alpha$ before hand. Choices of $\alpha$ may come from domain knowledge, but could also require a separate study.

In this work we assumed a fully adaptive adversarial contamination, constrained only by the total fraction of contamination at any time step.  By making more assumptions about the adversary, it is likely possible to improve uncontaminated regret bounds.

\paragraph{Limitations}
The adversary used in the simulation is quite simple and does not take full advantage of the power we allow in our model. We designed it as a first test of our algorithms and associated theory. In the future, we would like to design simulated adversaries that are modeled on real world contamination. It will also be important to deploy contamination robust algorithms in the real world. This will require thought on how to select various tuning parameters ahead of the deployment.

There remain many open questions in this area. In particular, we think this work could be improved along the following directions.%: 1) introducing randomized algorithms. Thompson sampling is typically chosen over UCB in practice, an equivalent robust randomized algorithm could be more practical for use than UCB based algorithms; 2) utilizing correlation between the contamination and rewards distributions. If the contaminated rewards contain information about the true rewards, this should be used to improve performance.

\paragraph{Randomized algorithms} UCB-type algorithms are often outperformed in applications by the randomized Thompson sampling algorithm. Creating a randomized algorithm that accounts for the contamination model would increase the practicality of this line of work.

\paragraph{Contamination correlated with true rewards} One possibility is that the contaminated rewards contain information of the true rewards. For example if contamination can be missing data, we know dropout can be correlated with the treatment condition.

\subsubsection*{Acknowledgements}
L.N. acknowledges the support of NSF via grant DMS-1646108 and thanks Joseph Jay Williams for helpful discussions and for inspiring this work.  A.T. would like to acknowledge the support of a Sloan Research Fellowship and NSF grant CAREER IIS-1452099.

%\paragraph{Contamination Detection} Previous methods for handling contamination rely on robust statistics or diluting the contamination. Because it is often the case that the sub-optimality gap in truth is quite small, a small amount of contamination could completely obscure the optimal action. It is difficult to detect and control these small shifts in the contaminated mean with contamination robust statistics.

%As an alternative approach, one can instead try and detect the contamination. With this method, any algorithm can be used without modification if the contaminated rewards can successfully be ignored.

%%%%%%%%%%%%%%%%%%%%%%%%%%%%%%%%%%%%%%%%%%%%%%%%%%%%%%%%%%%%%%%%%%%%%%
\subsubsection*{References}
\printbibliography[heading=none]

%%%%%%%%%%%%%%%%%%%%%%%%%%%%%%%%%%%%%%%%%%%%%%%%%%%%%%%%%%%%%%%%%%%%%%

\iftrue
\newpage
\onecolumn
\begin{appendices}

%%%%%%%%%%%%%%%%%%%%%%%%%%%%%%%%%%%%%%%%%%%%%%%%%%%%%%%%%%%%%%%%%%%%%%%
\section{Proofs}
%%%%%%%%%%%%%%%%%%%%%%%%%%%%%%%%%%%%%%%%

\subsection{Theorem \ref{thm:trMeanConcenG}}
\trMeanConcenG*
\begin{proof}[Proof of \cref{thm:trMeanConcenG}]
 \label{pf:trMeanConcenG}
 Without loss of generality assume $\mu = 0$ for the underlying true distribution. For $X \sim \sigma$-sub-Gaussian, by definition, we have:
 \begin{align*}
    &P\bigg( |X| \geq \mu + \eta \bigg)
     \leq 2\exp (-\frac{\eta^2}{2\sigma^2})\\
    &P \bigg(|\bar{x}_n - \mu| \geq \sigma \sqrt{\frac{2}{n}\log\frac{2}{\delta_1}} \bigg)
     \leq \delta_1
 \end{align*}
and
\begin{align*}
 &P\bigg(\max_{i \in [n]}|X_i| \geq t\bigg)
    \leq 2n\exp\bigg( -\frac{t^2}{2\sigma^2} \bigg)\\
 &P\bigg(\max_{i \in [n]}|X_i| \geq \sigma\sqrt{2\log\frac{2n}{\delta_2}}\bigg)
    \leq \delta_2.
\end{align*}
Let $\tilde{G} \subset G_n$ represent the points which are not contaminated and $C \subset G_n$ represent the contaminated points. Then our sample can be represented by the union $S_n = \tilde{G} \cup C$. Let $R$ represent the points that remain after trimming $\alpha$ fraction of the largest and smallest points, and $T$ be the set of points that were trimmed. Then we have that.
\begin{align*}
 |\text{trMean}_{\alpha}(S_n)| &= \bigg| \frac{1}{(1-2\alpha)n} \sum_{x \in R} x\bigg|\\
    &= \frac{1}{(1-2\alpha)n}
     \bigg|
         \sum_{x \in \tilde{G} \cap R} x
        + \sum_{x \in C \cap R} x
     \bigg|\\
    &\leq \frac{1}{(1-2\alpha)n}
     \bigg|
         \underbrace{\sum_{x \in \tilde{G}} x}_{A_1}
        - \underbrace{\sum_{x \in \tilde{G} \cap T} x}_{A_2}
        + \underbrace{\sum_{x \in C \cap R} x}_{A_3}
     \bigg|\\
    &\leq \frac{1}{(1-2\alpha)n} \bigg(
         \bigg| \underbrace{\sum_{x \in \tilde{G}} x}_{A_1} \bigg|
        + \bigg| \underbrace{\sum_{x \in \tilde{G} \cap T} x}_{A_2} \bigg|
        + \bigg| \underbrace{\sum_{x \in C \cap R} x}_{A_3} \bigg|
     \bigg)
\end{align*}
with
\begin{align*}
 A_1 &= \bigg| \sum_{x \in G_n} x  - \sum_{x \in G_n\setminus \tilde{G}} x\bigg|
 \leq \bigg| \sum_{x \in G_n} x \bigg| + \bigg| \sum_{x \in G_n\setminus \tilde{G}} x\bigg|
     \leq n|\bar{x}_{G_n}| + \varepsilon n \max_{x \in G_n} |x|
     &&\text{w.p. at least } 1-\delta_1 - \delta_2,\\
 A_2 &\leq 2 \alpha n \max_{x \in G_n} |x|   &&\text{w.p. at least } 1-\delta_2,\\
 A_3 &\leq \varepsilon n\max_{x \in G_n} |x|  &&\text{w.p. at least } 1-\delta_2.
\end{align*}
Combining we get,
\begin{align*}
 |\text{trMean}_{\alpha}(S_n) - \mu|
 &\leq \frac{1}{(1-2\alpha)} \bigg(|\bar{x}_{G_n}|
        + \max_{x \in G_n} |x|(2\varepsilon + 2\alpha) \bigg)\\
 &\leq \frac{1}{(1-2\alpha)} \bigg(|\bar{x}_{G_n}|
        + \max_{x \in G_n} |x|(4\alpha) \bigg)\\
 &\leq \frac{\sigma}{(1-2\alpha)} \bigg(
    \sqrt{\frac{2}{n} \log\frac{2}{\delta_1}}
    + 4\alpha\sqrt{2\log\frac{2t}{\delta_2}} \bigg)
\end{align*}

with probability at least $1-\delta_1 - \delta_2$. Letting $\delta_1 = \frac{2}{t^2}$ and $\delta_2 = \frac{2}{t^2}$, and assuming $\alpha \geq \varepsilon$, we have,
\begin{align*}
 |\text{trMean}_{\alpha}(S_n) - \mu| \leq \frac{\sigma}{(1-2\alpha)} \bigg(
    \sqrt{\frac{4}{n} \log(t)}
    + 4\alpha\sqrt{6\log(t)} \bigg)
\end{align*}
with probability at least $1 - \frac{4}{t^2}$.
\end{proof}

%%%%%%%%%%%%%%%%%%%%%%%%%%%%%%%%%%%%%%%%

\subsection{Theorem \ref{thm:sMeanConcenG}}
\sMeanConcenG*

\begin{proof}[Proof of \cref{thm:sMeanConcenG}]
 \label{pf:sMeanConcenG}
 Without loss of generality assume $\mu = 0$ for the underlying true distribution. Let $X \sim \sigma$-sub-Gaussian.

 We want to bound the impact of the contaminated points in our interval. Once we have this bound, the proof follows just as in the trimmed mean.

 Assume $\alpha < 1/3$ and $\varepsilon \leq \alpha$. Let $J$ be the interval that contains the shortest $1-\alpha$ fraction of $S_n$, $I$ be the interval that contains $\tilde{G}$ (i.e. the remaining good points after contamination), and $T$ be the interval that contains the points of $S_n$ after trimming the $\alpha$ largest and smallest fraction of points. Use $|I|$ to denote the length of interval $I$. It must be that $I \cap J \neq \emptyset$ because otherwise the points in $I\cup J$ would contain $2 - 2\alpha > 1$ fraction of $S_n$. Let $c$ be a point in $I \cap J$ and $x$ be a point in $J$. Recall that trMean$_{\alpha}(S_n)$ is the trimmed mean of the contaminated sample $S_n$ from above. Then we have,
 \begin{align*}
    |x| &\leq |x-c| + |c- \text{trMean}_{\alpha}(S_n)| + |\text{trMean}_{\alpha}(S_n)|\\
        &\leq |J| + |I| + |\text{trMean}_{\alpha}(S_n)|\\
        &\leq 2|I| + |\text{trMean}_{\alpha}(S_n)|
 \end{align*}
 The second step comes from $x$ and $c$ both being in $J$ and because $I \supseteq T$. The third step comes from $|J| \leq |I|$.

To bound the length of $I$ we have,
 \begin{align*}
    |I| &\leq 2 \max_{x \in G_n} |x| \ \ \text{w.p. at least } 1-\delta_2.
 \end{align*}
Finally, since
\[
 |\text{trMean}_{\alpha}(S_n)| \leq \frac{1}{(1-2\alpha)}(|\bar{x}_{G_n}|
     + 4\alpha \max_{x \in G_n} |x| )
\]
with probability at least $1-\delta_1 - \delta_2$, we get that for $x \in J$,
\begin{align*}
 |x| &\leq 4 \max_{x \in G_n} |x|
            +\frac{1}{(1-2\alpha)}(|\bar{x}_{G_n}|
            + 4\alpha \max_{x \in G_n} |x| )
     &&  \text{w.p. at least } 1-\delta_1 - \delta_2,\\
     &= \frac{|\bar{x}_{G_n}|}{1-2\alpha}
            + \Big(4 + \frac{4\alpha}{1-2\alpha}\Big)\max_{x \in G_n} |x|.
\end{align*}
Now that we have a bound on the contaminated points in $J$, our analysis follows as before,
\begin{align*}
 |\text{sMean}_{\alpha}&(S_n)| \\
        &\leq \frac{1}{(1-\alpha)n} \bigg(
         \bigg| \underbrace{\sum_{x \in \tilde{G}} x}_{A_1} \bigg|
        + \bigg| \underbrace{\sum_{x \in \tilde{G} \cap \neg J} x}_{A_2} \bigg|
        + \bigg| \underbrace{\sum_{x \in C \cap J} x}_{A_3} \bigg|
     \bigg)
\end{align*}
where
\begin{align*}
 A_1 &\leq n|\bar{x}_{G_n}| + \varepsilon n \max_{x \in G_n} |x|
     && \text{w.p. at least } 1-\delta_1 - \delta_2,\\
 A_2 &\leq \alpha n \max_{x \in G_n} |x|
     && \text{w.p. at least } 1-\delta_2,\\
 A_3 &\leq \varepsilon n \Bigg( \frac{|\bar{x}_{G_n}|}{1-2\alpha}
        + \Big(4 + \frac{4\alpha}{1-2\alpha}\Big)\max_{x \in G_n} |x| \Bigg)
     && \text{w.p. at least } 1-\delta_1 - \delta_2.
\end{align*}
Combining we get,
\begin{align*}
 &|\text{sMean}_{\alpha}(S_n) - \mu| \\
    &\leq \frac{1}{(1-\alpha)} \bigg(
     |\bar{x}_{G_n}|\big(1 + \frac{\varepsilon}{1-2\alpha}\big)
     + \max_{x \in G_n} |x| \big(5\varepsilon + \alpha
     + \frac{4\alpha\varepsilon}{1-2\alpha} \big)  \bigg)\\
    &\leq \frac{1}{(1-\alpha)} \bigg(
     |\bar{x}_{G_n}|\big(\frac{1-\alpha}{1-2\alpha}\big)
     + \max_{x \in G_n} |x| \big(6\alpha
     + \frac{4\alpha^2}{1-2\alpha} \big)  \bigg)\\
    &= \frac{1}{1-\alpha} \bigg(
     |\bar{x}_{G_n}| \big(\frac{1-\alpha}{1-2\alpha}\big)
     + \max_{x \in G_n} |x| \frac{6\alpha-8\alpha^2}{1-2\alpha}\bigg)\\
     %\hspace{90pt} \text{if } \alpha=\alpha\\
    %&= \frac{|\bar{x}_{G}|}{1-2\alpha}
    %  + \frac{4\alpha\max_{x \in G} |x| }{1-2\alpha}\\
    &\leq \frac{\sigma}{1-2\alpha} \sqrt{\frac{2}{n} \log\frac{2}{\delta_1}}
     + \frac{(6\alpha-8\alpha^2) \sigma}{(1-2\alpha)(1-\alpha)} \sqrt{2\log\frac{2t}{\delta_2}}
\end{align*}
With probability at least $1-\delta_1 - \delta_2$. Letting $\delta_1 = \frac{2}{t^2}$ and $\delta_2 = \frac{2}{t^2}$, and assuming $\alpha \geq \varepsilon$, we have,
\begin{align*}
 |\text{sMean}_{\alpha}&(S_n) - \mu|\\
    &\leq \frac{\sigma}{1-2\alpha} \sqrt{\frac{4}{n} \log t}
     + \frac{(6\alpha-8\alpha^2) \sigma}{(1-2\alpha)(1-\alpha)} \sqrt{6\log t}
\end{align*}
With probability at least $1 - \frac{4}{t^2}$.
\end{proof}
%%%%%%%%%%%%%%%%%%%%%%%%%%%%%%%%%%%%%%%%
\subsection{Theorem \ref{thm:trMeanRegretG}}
\trMeanRegretG*
\begin{proof}[Proof of \cref{thm:trMeanRegretG}]
 \label{pf:trMeanRegretG}
 First will show that $\EX[N_a(t)] < \infty$ for non-optimal actions. Assume $N_a(t) \geq \frac{64\sigma_0^2\log(T)}{\Delta_a^2}$.
 \begin{align*}
    &\hat{\mu}_{a}
     + \frac{\sigma_0}{(1-2\alpha)}\bigg(\sqrt{\frac{4}{N_a(t)}\log t}
     + 4\alpha\sqrt{6\log(t)}\bigg)\\
    &\leq \mu_{a}
     + \frac{\sigma_i+\sigma_0}{(1-2\alpha)}\bigg(\sqrt{\frac{4}{N_a(t)}\log t}
     + 4\alpha\sqrt{6\log(t)}\bigg)
     && \text{w.p. at least } 1-\frac{4}{t^2}\\
    &\leq \mu^*
     - \Delta_a
     + \frac{2\sigma_0}{(1-2\alpha)}\bigg(\sqrt{\frac{4}{N_a(t)}\log t}
     + 4\alpha\sqrt{6\log(t)}\bigg)\\
    &\leq \mu^*
     - \Delta_a
     + \frac{\Delta_a}{2(1-2\alpha)}
     + \frac{2\sigma_0 4\alpha}{(1-2\alpha)} \sqrt{6\log t}
     && N_a(t) \geq \frac{64\sigma_0^2\log(T)}{\Delta_a^2}\\
    &\leq \mu^*
     && \alpha \leq \frac{\Delta_a}{ 4(\Delta_a + 4\sigma_0\sqrt{6\log(t))}}\\
    &\leq \hat{\mu}^*
     + \frac{\sigma_{i^*}}{(1-2\alpha)}\bigg(\sqrt{\frac{4}{N^*(t)}\log t}
     + 4\alpha\sqrt{6\log(t)} \bigg)
     && \text{w.p. at least } 1-\frac{4}{t^2}\\
    &\leq \hat{\mu}^*
     + \frac{\sigma_0}{(1-2\alpha)}\bigg(\sqrt{\frac{4}{N^*(t)}\log t}
     + 4\alpha\sqrt{6\log(t)}\bigg).
 \end{align*}
 Now to find $\EX[N_a(T)]$ for non-optimal actions.
 \begin{align*}
    \EX[N_a(T)] &= 1 + \EX \bigg[ \sum_{t=K+1}^T \1\{A_t=a\} \bigg]\\
    &= 1 + \EX \bigg[ \sum_{t=K+1}^T \1 \bigg\{A_t=a, N_a(t) \leq \frac{64\sigma_0^2\log(T)}{\Delta_a^2} \bigg\}
     + \1 \bigg\{A_t=a, N_a(t) > \frac{64\sigma_0^2\log(T)}{\Delta_a^2}  \bigg\} \bigg]\\
    &\leq 1 + \frac{64\sigma_0^2\log(T)}{\Delta_a^2}
     + \sum_{t=K+1}^T \mathbb{P} \bigg[ A_t=a, N_a(t) > \frac{64\sigma_0^2\log(T)}{\Delta_a^2}  \bigg]\\
    &= 1 + \frac{64\sigma_0^2\log(T)}{\Delta_a^2}
     + \sum_{t=K+1}^T \mathbb{P} \bigg[ A_t=a | N_a(t) > \frac{64\sigma_0^2\log(T)}{\Delta_a^2}  \bigg]
     \mathbb{P} \bigg[ N_a(t) > \frac{64\sigma_0^2\log(T)}{\Delta_a^2}  \bigg]\\
    &\leq 1 + \frac{64\sigma_0^2\log(T)}{\Delta_a^2}  + \sum_{t=K+1}^T \frac{8}{t^2}\\
    &\leq \frac{64\sigma_0^2\log(T)}{\Delta_a^2}  + 15.
 \end{align*}

 Finally, we can find the regret following the standard analysis,

 \begin{align*}
    \bar{R} &= \sum_{a=2}^K \Delta_a\EX[N_a(T)] &&\\
    &= \sum_{\Delta_a < \Delta} \Delta_a\EX[ N_a(T)]
     + \sum_{\Delta_a \geq \Delta} \Delta_a\EX \bigg[ N_a(T) \bigg] \\
    &\leq \Delta T
     + \sum_{\Delta_a \geq \Delta} \big[ \frac{ 64\sigma_0^2\log(T)}{\Delta_a}
     + 15\Delta_a \big]
     && \EX[N_a(t)] \leq \frac{64\sigma_0^2\log(T)}{\Delta_a}  + 15\\
    &\leq 8\sigma_0\sqrt{KT\log(T)} + \sum 15\Delta_a
     && \Delta = \sqrt{\frac{64K\sigma_0^2\log(T)}{T}}.
 \end{align*}
\end{proof}

\subsection{Corollary \ref{cor:tr}}
\trMeanb*
\begin{proof}[Proof of \cref{cor:tr}]
 \label{pf:tr}
 By replacing the part of the concentration bound for the trimmed mean that is based on the maximum value in the sample with $b$, we get that,
\begin{align*}
 |\text{trMean}_{\alpha}(S_n) - \mu| \leq \frac{\sigma}{(1-2\alpha)}
    \sqrt{\frac{4}{n} \log(t)}
    + \frac{4\alpha}{1-2\alpha}b
\end{align*}
with probability at least $1 - \frac{4}{t^2}$.

First will show that $\EX[N_a(t)] < \infty$ for non-optimal actions. Assume $N_a(t) \geq \frac{64\sigma_0^2\log(T)}{\Delta_a^2}$.
 \begin{align*}
    &\hat{\mu}_{a}
     + \frac{\sigma_0}{(1-2\alpha)}\sqrt{\frac{4}{N_a(t)}\log t}
     + \frac{4\alpha}{1-2\alpha}b\\
    &\leq \mu_{a}
     + \frac{\sigma_i+\sigma_0}{(1-2\alpha)}\sqrt{\frac{4}{N_a(t)}\log t}
     + \frac{8\alpha}{1-2\alpha}b
     && \text{w.p. at least } 1-\frac{4}{t^2}\\
    &\leq \mu^*
     - \Delta_a
     + \frac{2\sigma_0}{(1-2\alpha)}\sqrt{\frac{4}{N_a(t)}\log t}
     + \frac{8\alpha}{1-2\alpha}b\\
    &\leq \mu^*
     - \Delta_a
     + \frac{\Delta_a}{2(1-2\alpha)}
     + \frac{8\alpha}{(1-2\alpha)} b
     && N_a(t) \geq \frac{64\sigma_0^2\log(T)}{\Delta_a^2}\\
    &\leq \mu^*
     && \alpha \leq \frac{\Delta_a}{4(\Delta_a+4b)}\\
    &\leq \hat{\mu}^*
     + \frac{\sigma_{i^*}}{(1-2\alpha)}\sqrt{\frac{4}{N^*(t)}\log t}
     + \frac{4\alpha}{1-2\alpha}b
     && \text{w.p. at least } 1-\frac{4}{t^2}\\
    &\leq \hat{\mu}^*
     + \frac{\sigma_0}{(1-2\alpha)}\sqrt{\frac{4}{N^*(t)}\log t}
     + \frac{4\alpha}{1-2\alpha}b.
 \end{align*}
 Results follow with a similar analysis as above.
\end{proof}

%%%%%%%%%%%%%%%%%%%%%%%%%%%%%%%%%%%%%%%

\subsection{Theorem \ref{thm:sMeanRegretG}}
\sMeanRegretG*

\begin{proof}[Proof of \cref{thm:sMeanRegretG}]
 \label{pf:sMeanRegretG}
 The proof for the contamination robust UCB using the $\alpha$-shorth mean is similar to that of the trimmed mean.

 \begin{align*}
    &\hat{\mu}_{a}
     + \frac{\sigma_0}{1-2\alpha} \sqrt{\frac{4}{N_a(t)} \log t}
     + \frac{(6\alpha-8\alpha^2) \sigma}{(1-2\alpha)(1-\alpha)} \sqrt{6\log t}\\
    &\leq \mu^*
     - \Delta_a
     + \frac{2\sigma_0}{1-2\alpha} \sqrt{\frac{4}{N_a(t)} \log t}
     + 2\frac{(6\alpha-8\alpha^2) \sigma_0}{(1-2\alpha)(1-\alpha)} \sqrt{\log t}
     && \text{w.p.a.l } 1-\frac{4}{t^2}\\
    &\leq \mu^*
     - \Delta_a
     + \frac{\Delta_a}{2(1-2\alpha)}
     + \frac{18\alpha\sigma_0}{(1-2\alpha)} \sqrt{6\log t}
     && N_a(t) \geq \frac{64\sigma_0^2\log(t)}{\Delta_a^2}, \ \alpha < 1/3\\
    &\leq \mu^*
     && \alpha \leq \frac{\Delta_a}{4(\Delta_a + 9\sigma_0\sqrt{6 \log t})}\\
    &\leq \hat{\mu}^*
    + \frac{\sigma_0}{1-2\alpha} \sqrt{\frac{4}{N^*(t)} \log t}
    + \frac{6\alpha-8\alpha^2 \sigma}{(1-2\alpha)(1-\alpha)} \sqrt{6\log t}
 \end{align*}

Using the analysis from the trimmed mean regret, we again get,
 \[
    \EX[N_a(t)] \leq \frac{64\sigma_0^2\log T}{\Delta_a} + \sum 15 \Delta_a
 \]

 Using this value and standard regret analysis yields

 \[
  \bar{R}_T \leq 8\sigma_0\sqrt{KT\log(T)} + \sum 15\Delta_a.
 \]
\end{proof}

\subsection{Corollary \ref{cor:sh}}
\shMeanb*
\begin{proof}[Proof of \cref{cor:sh}]
 \label{pf:sh}
 By replacing the part of the concentration bound for the trimmed mean that is based on the maximum value in the sample with $b$, we get that,

 \begin{align*}
 |\text{sMean}_{\alpha}(S_n) - \mu|
    \leq \frac{\sigma}{1-2\alpha} \sqrt{\frac{4}{n} \log t}
     + \frac{6\alpha-8\alpha^2 }{(1-2\alpha)(1-\alpha)} b
\end{align*}
With probability at least $1 - \frac{4}{t^2}$.

Follow similar analysis as in \cref{pf:tr} but setting constraint to be,
\begin{align*}
    \varepsilon &\leq \alpha \leq \frac{\Delta_{\min}}{4(\Delta_{\min}+9b)}
\end{align*}
\end{proof}

%%%%%%%%%%%%%%%%%%%%%%%%%%%%%%%%%%%%%%%%%%%%%%%%%%%%%%%%%%
\section{Relationship of $\varepsilon$ and $\Delta_{\min}$}
\label{sec:deltamin}

One quick example showing that $\varepsilon > \Delta_{\min}$ can prohibit sublinear regret is to consider the CSB game with two actions and Bernoulli rewards. If $a_1 \sim B(p)$ and $a_2 \sim B(p-\varepsilon)$ then an adversary can choose all the contaminated rewards for $a_2$ to be 1 making it appear that $a_2 \sim B(p)$. Thus the actions are indistinguishable to the learner.
%%%%%%%%%%%%%%%%%%%%%%%%%%%%%%%%%%%%%%%%%%%%%%%%%%%%%%%%%%

However, we can still provide a bound for larger values of $\varepsilon$ provided one is willing to tolerate a linear term in the regret. We outline the argument only for the trimmed mean case since the argument for the shorth mean is very similar. Note that argument for bounding $\mathbb{E}[N_a(T)]$ in \Cref{thm:trMeanRegretG} works under the condition
\[
\alpha \leq \frac{\Delta_a}{ 4(\Delta_a + 4\sigma_0\sqrt{6\log(T))}} .
\]
Let $\mathcal{S}$ be the set of actions satisfying this condition. The arguments in the proof of \Cref{thm:trMeanRegretG} show that
\[
\sum_{a>1, a \in \mathcal{S}} \Delta_a \mathbb{E}[N_a(T)]
\le 8\sigma_0\sqrt{KT\log(T)} + \sum_{a>1, a \in \mathcal{S}} 15\Delta_a .
\]
Therefore the bound of $\tilde{O}(\sigma_0\sqrt{KT})$ holds only for the regret due to actions $a \in \mathcal{S}$. For any action $a \notin \mathcal{S}$, we have
\[
\Delta_a < \frac{16 \alpha \sigma_0\sqrt{6 \log(T)}}{1-4\alpha}
\]
assuming $\alpha < 0.25$. The total regret contribution for $a \notin \mathcal{S}$ is therefore
\begin{align*}
   \sum_{a>1, a \notin \mathcal{S}} \Delta_a \mathbb{E}[N_a(T)]
    &\le \frac{16 \alpha \sigma_0\sqrt{6 \log(T)}}{1-4\alpha}
    \sum_{a>1, a \notin \mathcal{S}} \mathbb{E}[N_a(T)]\\
    &\le \frac{16 \alpha \sigma_0\sqrt{6 \log(T)}}{1-4\alpha} T
\end{align*}{}

So the total regret is $\tilde{O}(\sqrt{KT}+\tfrac{\alpha}{1-4\alpha} T)$.
\end{appendices}
\fi
\end{document}